\documentclass{article}

\usepackage[accepted]{./style/icml_2024}
\PassOptionsToPackage{numbers,compress}{natbib}
\usepackage{natbib}  
\usepackage{amsmath,amsfonts,bm,amssymb,nicefrac,algorithm,algpseudocode}
\usepackage{amsthm}
\usepackage{booktabs}          
\usepackage{graphicx}          
\usepackage{subcaption}        
\usepackage{multirow}        
\usepackage{stmaryrd}        
\usepackage{yhmath}        

\def\eqref#1{Eq.~\ref{#1}}   

\def\vepsilon{{\bm{\epsilon}}}
\def\vzero{{\bm{0}}}

\def\vmu{{\bm{\mu}}}

\def\va{{\bm{a}}}

\def\vx{{\bm{x}}}
\def\vy{{\bm{y}}}

\def\mI{{\bm{I}}}

\def\mSigma{{\bm{\Sigma}}}

\newcommand{\E}{\mathbb{E}}

\newcommand{\R}{\mathbb{R}}

\newtheoremstyle{mythmstyle} 
{\topsep}    
{\topsep}    
{\itshape}   
{0pt}        
{\bfseries}  
{}           
{ }          
{}           
\theoremstyle{mythmstyle}

\newtheorem{lemma}{Lemma}

\newtheorem{proposition}{Proposition}

\newtheorem*{proposition*}{Proposition}
\newtheorem*{theorem*}{Theorem}

\newcommand{\cN}{\mathcal N}

\newcommand{\KL}{\mathop{\mathrm{KL}}\nolimits}

\usepackage{physics}
\usepackage{hyperref}
\usepackage{cleveref}

\definecolor{grape}{rgb}{0.43, 0.17, 0.71}
\definecolor{blush}{rgb}{0.87, 0.36, 0.51}
\definecolor{aquua}{HTML}{689d6a}

\usepackage[textwidth=2cm]{todonotes}  

\begin{document}

\twocolumn[
\icmltitle{A Practical Diffusion Path for Sampling}

\icmlsetsymbol{equal}{*}

\begin{icmlauthorlist}
\icmlauthor{Omar Chehab}{ensae}
\icmlauthor{Anna Korba}{ensae}
\end{icmlauthorlist}

\icmlaffiliation{ensae}{ENSAE, CREST, IP Paris, France}

\icmlcorrespondingauthor{Omar Chehab}{emir.chehab@ensae.fr}
\icmlkeywords{Machine Learning, ICML}

\vskip 0.3in]

\printAffiliationsAndNotice{}

\begin{abstract}

Diffusion models are state-of-the-art methods in generative modeling when samples from a target probability distribution are available, and can be efficiently sampled, using score matching to estimate score vectors guiding a Langevin process. However, in the setting where samples from the target are not available, e.g. when this target's density is known up to a normalization constant, the score estimation task is challenging. Previous approaches rely on Monte Carlo estimators that are either computationally heavy to implement or sample-inefficient. 
In this work, we propose a computationally attractive alternative, relying on the so-called dilation path, that yields score vectors that are available in closed-form. This path interpolates between a Dirac and the target distribution using a convolution. 
We propose a simple implementation of Langevin dynamics guided by the dilation path, using adaptive step-sizes. We illustrate the results of our sampling method on a range of tasks, and shows it performs better than classical alternatives.

\end{abstract}



\section{Introduction}

\paragraph{} Drawing samples from a target distribution is a key problem in statistics. In many settings, the target distribution is known up to a normalizing constant. This is often the case for pre-trained energy-based probabilistic models~\citep[Chapter 24]{murphy2023mlbook}, whose parameters have already been estimated~\citep{hyvarinen2005scorematching,gutmann2012nce,hinton2002contrastivedivergence,gao2020fce}. Another example comes from Bayesian statistics, where the posterior model over parameters given observed data is classically known only up to a normalizing constant~\citep[Chapter 11]{wasserman2010parametric}. 

\paragraph{} Classical methods for sampling from such target distributions, such as simulating a Langevin process with particles, are known to struggle when the target has many modes. Typically, the particles are first drawn to certain modes and then take exponential time to find all other modes~\citep{bovier2000metastability,bovier2004metastability,bovier2005metastability}. Many successful, alternative methods rely on a path of distributions, chosen by the user to steer the sampling process to reach all the modes and hopefully converge faster~\citep{neal2001annealed,geyer1991markov,marinari1992simulatedtempering,dai2020smcreview}. An instance of such a method is Annealed Langevin Dynamics ~\citep{song2019scorebasedmodel,song2020scorerecommendations}, that are simple to implement and are popular in Bayesian inference~\citep[Eq. 2.4]{dai2020smcreview}, in global optimization~\citep{geman1986simulatedannealing}, and more recently in sampling from high-dimensional image distributions with many modes~\citep{song2019scorebasedmodel,song2020scorerecommendations,song2020diffusion}. In the latter application, Annealed Langevin Dynamics have achieved state-of-the-art results by following a specific path of distributions, obtained by interpolating the multi-modal target and a standard Gaussian distribution with a convolution. 

Yet, despite its promising geometry, this convolutional path of distributions is not readily usable for Annealed Langevin Dynamics, whose implementation requires the score vectors of the path which are not available in closed-form. A number of estimators for these score vectors have recently been developed when the target distribution is only known by its unnormalized density, yet these estimators can be computationally heavy~\citep{huang2024reversediffusionrecursive} or sample-inefficient~\citep{huang2024reversediffusion}. 

\paragraph{} In this work, we introduce the dilation path, which is a limit case of the popular convolutional path in which the score vectors are available in closed-form. Our approach circumvents alternatives that require Monte Carlo simulation~\citep{huang2024reversediffusion,he2024reversediffusion,grenioux2024reversediffusion,saremi2024score,akhound2024reversediffusion} and is instead exceedingly simple to implement.

\section{Background}

\paragraph{Langevin dynamics}
The Unadjusted Langevin Algorithm (ULA) is a classical algorithm to draw samples from a target distribution $\pi$. It is written as noisy gradient ascent 
\begin{align}
    \vx_{k+1} 
    = 
    \vx_k
    +
    h_k \nabla \log \pi(\vx_k) 
    + 
    \sqrt{2 h_k} \vepsilon_k
    \label{eq:langevin_sampler}
\end{align}
with step sizes $h_k > 0$ and random noise $\vepsilon_k \sim \cN(\vzero, \mI)$. It can be viewed as a time-discretization of a Langevin diffusion~\citep{borodin2017stochastic}, where time is defined as $t = k h_k$ in the limit of null step sizes $h_k \rightarrow 0$.  
The rate of convergence is determined by constants (\textit{e.g.} Log-Sobolev, Poincaré) that describe the geometry of the target distribution~\citep{vempala2019rapid}: for multimodal distributions such as Gaussian mixtures, the constant degrades exponentially fast with the distance between modes~\citep{holley1987logsobolevineq,arnold2000logsobolevineq}, making convergence with given precision too slow to occur in a reasonable number of iterations. Yet, Langevin dynamics remain a popular choice for their computational simplicity: simulating ~\eqref{eq:langevin_sampler} requires computing the score vector $\nabla \log \pi(\vx_k)$ which does not depend on the target's normalizing constant. Hence, the Langevin sampler can be used to sample target distributions whose normalizing factor is unknown, and this property is unique among a broad class of samplers~\citep{chen2023gradientflows}.

\paragraph{Annealed Langevin dynamics} Many heuristics broadly known as annealing or tempering, consist in using the Langevin dynamics to sample from a path of distributions $(\mu_t)_{t \in \R_{+}}$ instead of the single target $\pi$, in hope that these intermediate distributions decompose the original sampling problem into easier tasks for Langevin dynamics. 
We specifically consider
\begin{align}
    \vx_{k+1} 
    = 
    \vx_k 
    +
    h_k \nabla \log \mu_k(\vx_k) 
    + 
    \sqrt{2 h_k} \vepsilon_k,
    \label{eq:annealed_langevin_sampler}
\end{align}
where the target now moves with time. This process is known as Annealed Langevin Dynamics~\citep{song2019scorebasedmodel}, and is sometimes combined with other sampling processes based on resampling~\citep[Eq. 2.4]{dai2020smcreview} or that directly simulate the path $(\mu_t)_{t \in \R_{+}}$~\citep[Appendix G]{song2020diffusion}.

\paragraph{Convolutional path}
Recently, a path obtained by taking the convolution of the target distribution $\pi$ and an easier, proposal distribution $\nu$
\begin{align}
    \mu_t(\vx) 
    &= 
    \frac{1}{\sqrt{1-\lambda_t}} \nu \left(
    \frac{\vx}{\sqrt{1 - \lambda_t}}
    \right) 
    \ast
    \frac{1}{\sqrt{\lambda_t}} \pi \left(
    \frac{\vx}{\sqrt{\lambda_t}}
    \right)
    \label{eq:convolutional_path}
\end{align}
has produced state-of-the-art results in sampling from high-dimensional and multimodal distributions~\citep{song2019scorebasedmodel}. 
Here, $\nu$ is typically a standard Gaussian and $\lambda: \R_{+} \rightarrow [0, 1]$ in an increasing function called schedule~\citep{chen2023noiseschedule}: popular choices use exponential $\lambda_t = \min(1, e^{-2(T-t)})$ for some fixed $T \geq 0$, or linear $\lambda_t = \min(1, t)$ functions~\citep[Table 1]{gao2023gaussianinterpolant}. Note that the exponential schedule is initialized at $\lambda_0 = e^{-2T}$: choosing $T$ to be big (resp. small) initializes the path of guiding distributions nearer to the proposal (resp. target). Recent work has empirically observed that this convolutional path may have a more favorable geometry for the Langevin sampler than another well-established path~\citep{phillips2024particlesampler}, obtained by taking the geometric mean of the proposal and target distributions~\citep{neal1998annealing,gelman1998importancesamplingext,dai2020smcreview}. 
However, using the convolutional path in practice requires computing the score vectors $\nabla \log \mu_t(\cdot)$ which has been at the center of recent work.

\paragraph{Computing the score with access to samples}  In machine learning literature, the target distribution is often accessed through samples $\vx_{\pi} \sim \pi$ only. These can be interpolated with samples from the proposal distribution $\vx_{\nu} \sim \nu$ to produce samples from the intermediate distributions $\mu_t$~\citep{albergo2023interpolant},
\begin{align}
    \vx_{t} 
    &=
    \sqrt{1 - \lambda_{t}} \vx_{\nu}
    + 
    \sqrt{\lambda_{t}} \vx_{\pi}
    \enspace .
\end{align}
These samples can then be used to evaluate loss functions whose minimizers are estimators of the scores $\nabla \log \mu_t$~\citep{hyvarinen2005scorematching,song2019scorebasedmodel}. Recent work provides theoretical guarantees that the estimation error deterioriates favorably, that is polynomially as opposed to exponentially, with the dimensionality of the data and separation between modes~\citep{qin2023annealedscorematching}.

\paragraph{Computing the score with access to an unnormalized density} In classical statistics, it is instead assumed that the target distribution is accessed through its unnormalized density, not its samples. A novel way to compute the scores has been at the center of recent efforts to use the convolutional path in this setup~\citep{huang2024reversediffusion,he2024reversediffusion,grenioux2024reversediffusion,saremi2024score,akhound2024reversediffusion}. Overwhelmingly, these works use an explicit Monte Carlo estimator of the score
\begin{align}
    &\nabla \log \mu_t(\vx)
    =
    \frac{e^{-(T-t)}}{1 - e^{-2(T-t)}}
    \E_{\vy \sim m_t} \bigg[
    (\vy - e^{T-t} \vx)
    \bigg]
    \label{eq:monte_carlo_score}
    , \\
    &m_t(\vy | \vx) \propto 
    \pi(\vy)
    \times
    \mathcal{N}(\vy; e^{T-t}\vx, \sqrt{e^{2(T-t)} - 1} \mI
    )
    \label{eq:monte_carlo_sampling_distribution}
\end{align}
obtained by replacing the expectation with an average over finite samples. These samples are drawn from a blurred version of the target $m_t(\vy | \vx)$, specifically the  (normalized) product of the target and proposal distributions. 

Using this estimator presents three important challenges:
\begin{enumerate}

    \item The number of sampling procedures

    Each query of the score~\eqref{eq:monte_carlo_score} at a certain $\vx$ requires running a new sampling procedure~\eqref{eq:monte_carlo_sampling_distribution}. For example, each run of the algorithm~\eqref{eq:annealed_langevin_sampler} will query the score function at each iteration, and will require running as many sampling procedures. 

    \item The complexity of sampling procedures

    Standard sampling procedures for~\eqref{eq:monte_carlo_sampling_distribution} are slow to converge. For example, ~\citet{he2024reversediffusion} use the rejection sampling algorithm whose computational cost is exponential in the dimension. Alternatively, the Unadjusted Langevin Algorithm~\eqref{eq:langevin_sampler} is guaranteed to be fast-converging when the sampled distribution~\eqref{eq:monte_carlo_sampling_distribution} is log-concave, which is the case whenever the Gaussian distribution dominates the target. This is achieved by a small enough $T$ in Eq.~\ref{eq:monte_carlo_score}-~\ref{eq:monte_carlo_sampling_distribution}, or equivalently by an initial schedule $\lambda_0 = e^{-2T}$ close enough to one so that the sampling process starts near the target~\citep{huang2024reversediffusion}. Some works suppose that small enough $T$ can be found as as a hyperparameter of the problem~\citep{huang2024reversediffusionrecursive,grenioux2024reversediffusion}. Using that value, \citet{grenioux2024reversediffusion} estimate the score~\eqref{eq:monte_carlo_score} over a window near the target  $\lambda_t \in [e^{-2T}, 1]$. To estimate the score nearer the proposal, \citet{huang2024reversediffusionrecursive} use the distribution at $e^{-2T}$ as the new target, and repeat. The computational complexity of such a procedure is prohitive: it scales exponentially in the number of windows (we verify this in Appendix~\ref{app:sec:complexity_calculation}).
    Some sampling methods will altogether hide the difficulty of sampling a non log-concave distribution~\eqref{eq:monte_carlo_sampling_distribution} by supposing access to an oracle~\citep{lee2021proximalsampling,chen2024proximalsampling}.

    \item The complexity of the estimator

    Even if we were able to efficiently sample~\eqref{eq:monte_carlo_sampling_distribution}, the estimator~\eqref{eq:monte_carlo_score} introduces an estimation error in the sampling process that can scale exponentially with the dimensionality of data points~\citep{huang2024reversediffusion}. 
    
\end{enumerate}

Importantly, these three challenges disappear when the score vectors $\nabla \log \mu_t(\cdot)$ are analytically computable. Finding a path that has both the favorable the geometry of the convolutional path and analytically computable score vectors is an unresolved problem.

\section{The dilation path}

We propose to use the limit of a Dirac proposal, anticipating that it will simplify the convolution which defines the path and consequently the score. We call the corresponding path a dilation path
\begin{align}
    \mu_t(\vx) 
    & = 
    \frac{1}{\sqrt{\lambda_t}} \pi \left( \frac{\vx}{\sqrt{\lambda_t}} \right)
    \enspace .
\end{align}
Notably, the scores are now analytically available
\begin{align}
    \nabla \log \mu_t(\vx) 
    &=
    \frac{1}{\sqrt{\lambda_t}} 
    \nabla \log \pi \left(
    \frac{\vx}{\sqrt{\lambda_t}}
    \right)
    \enspace .
\end{align}
We therefore recommend this path for practitioners and verify its simplicity in the experiments of section~\ref{sec:experiments}.
We also note that considering a Dirac proposal has been known to simplify the analytical equation of a path from another problem known as Dynamic Optimal Transport or Schr{\"o}dinger bridge. In that setup, the goal is to find a path that looks random (is close in Kullback-Leibler divergence to a Wiener process) but with fixed distributions at the start and end (proposal and target). Considering a Dirac proposal simplifies the equations of that path which is then given the name of F\"{o}llmer path~\cite{ding2023follmerflow,huang2021follmer,jiao2021follmerconvergence}.

To better understand the geometry of the dilation path, we write the convolutional path between the proposal and target distributions assuming a fairly general parametric family and then consider the special case of a Dirac proposal. We 
recall the following in Appendix~\ref{app:parametric_families}
\begin{proposition}[Gaussian mixture parametric family]    
    \label{proposition:gaussian_mixture}
    Suppose the proposal is a Gaussian $\nu := \mathcal{N}(\cdot, \vzero, \mSigma_0)$  and the target is a mixture of $M$ Gaussians with means $(\vmu_m)_{m \in \llbracket 1, M \rrbracket}$, covariances $(\mSigma_m)_{m \in \llbracket 1, M \rrbracket}$, and positive weights $(w_m)_{m \in \llbracket 1, M \rrbracket}$ that sum to one.
    
    The convolutional path ~\eqref{eq:convolutional_path} produces distributions $\pi_t$ that conveniently remain in the Gaussian mixture parametric family. Their weights are constant
    $w_m(t) = w_m$ and their means and covariances interpolate additively the proposal's and target's, as $\vmu_m(t) = \sqrt{\lambda_t} \vmu_m$ and $\mSigma_m(t) = (1-\lambda_t) \mSigma_0 + \lambda_t \mSigma_m$.

    In the case of a Dirac proposal, we have instead $\mSigma_m(t) = \lambda_t \mSigma_m$.
\end{proposition}
Along the interpolation between a Dirac and the target, the dilation path remains a mixture with constant weights, and only the means and covariances are updated. Preserving the mode weights along the path is a desirable feature for at least two reasons. First, it has been observed that the score vector used to simulate~\eqref{eq:annealed_langevin_sampler} is known to be rather blind to mode weights when the modes are far apart~\citep{wenliang202scoreandmodeweights}. Second, it is reported that Langevin dynamics are slow-converging when initial and target weights differ, which is sometimes referred to as ``mode switching"~\citep{phillips2024particlesampler}.

\paragraph{Numerical implementation} 
A naive implementation of the annealed Langevin sampler in~\eqref{eq:annealed_langevin_sampler} with the dilation path is numerically unstable:  distributions $p_{\lambda_t}$ when $\lambda_t$ is close to zero have steep modes around which the gradient is numerically infinite. We verified in simple simulations that the particles diverge. To mitigate this effect, we use the effective numerical trick of an adaptive step size $h$ to control the magnitude of the gradient term $h \nabla \log \pi_k(\vx)$ in~\eqref{eq:annealed_langevin_sampler}. 

Recent works have adapted the step size to time, so that $h_k \propto 1 / \E[\|\nabla \log \pi_{k}(\vx_k) \|^2]$, empirically observing that it decreases with the number of iterations \citep[Section 4.3]{song2019scorebasedmodel} \citep[Section 3.3]{song2020scorerecommendations}. Their motivation is to roughly equalize magnitudes of the gradient and noise terms in \eqref{eq:annealed_langevin_sampler}: $ h_k \|\nabla \log \pi_k(\vx_k) \| \approx \sqrt{2 h_k} \| \epsilon_k \| \approx 1 / \E[\|\nabla \log \pi_k(\vx_k) \|]$. In practice, they use a proxy for $\E[\|\nabla \log \pi_{k}(\vx_k) \|^2]$ which is not computed. 

Instead, we use a finer adaptation, where the step size depends on time \textit{as well as} the current position of a particle $h_k(\vx_k) \propto 1 / \|\nabla \log \pi_k(\vx_k) \|$, or a bounded version in practice~\citep[Eq. 3.1]{leroy2024langevinadaptivestep}. Our motivation is to normalize the magnitude of the gradient term in~\eqref{eq:annealed_langevin_sampler}. Note that adapting the step size in this way alters the stationary distribution of the sampling process. A corrective step can be used~\citep[Eq. 2.14]{leroy2024langevinadaptivestep} but it involves a Hessian $\nabla^2 \log \pi(\cdot)$ which is expensive to compute in high dimensions and is outside the scope of this paper.

\section{Experiments}
\label{sec:experiments}

In this section, we numerically verify convergence using the dilation path. We constrast it with another well-established path --- obtained as the geometric mean of the proposal and target ---  whose scores are also available so that the computational budget for running Annealed Langevin Dynamics~\eqref{eq:annealed_langevin_sampler} is comparable. 

\paragraph{Convergence metrics}
To measure the discrepancy between the distribution the particles $p_k$ realizing the sampling process \eqref{eq:annealed_langevin_sampler} and the target distribution $\pi$, we use (approximations of) a number of statistical divergences. Among them, we distinguish the Kernel Stein Discrepancy (KSD)~\citep{liu2016ksd,chwialkowski2016ksd,mackey2017ksd}, defined as
\begin{align}
    \mathrm{KSD}^2(p_k, \pi)
    & = 
    \E_{(\vx, \vx') \sim 
    p_k \otimes p_k} [
    K(\vx, \vx'; \nabla \log \pi, K')
    ]
\end{align}
where $K$ and $K'$ are kernels defined in Appendix~\ref{app:sec:experiments}.
Importantly, the KSD can be approximated using what is available:  samples from $p_k$ and the score of the target $\pi$. 

The other statistical divergences we use are abbreviated as (KL, revKL, MMD, OT) and are defined in Appendix~\ref{app:ssec:statistical_divergences}. Their approximations require access to samples from both the process $p_k$ and the target $\pi$, which is not realistic as samples from the target are unavailable in practice. However, in our synthetic experiments, we are able to track these metrics. 

We note many of these statistical divergences can be blind to mode coverage, which means that a sampling process can find few modes of the target $\pi$ while ignoring other modes, and still produce low values of these metrics. In particular, this has been noted for the KSD~\citep{korba2021kernel, benard2023ksdmodeblindness} and the revKL~\citep{verine2023precisionrecall}. We therefore introduce a metric which we call the Multimodality Score (MMS), that specifically measures mode coverage. The MMS is defined as the root mean squared error between the actual and expected number of particles per mode.

\paragraph{Sampling a Gaussian mixture}
We follow the setups of~\citet{zhang2020cyclicalmcmc} and ~\citet{midgley2023flowais} where the target distribution is a Gaussian mixture with 16 and 40 modes respectively. We use a standard Gaussian as the proposal distribution, except for the dilation path which has a Dirac proposal by design. Results are reported in Figure~\ref{fig:gaussian_mixtures} and additional convergence diagnostics are reported in Appendix~\ref{app:sec:experiments}. 

\begin{figure}[!ht]
\centering
    \includegraphics[width=\linewidth]{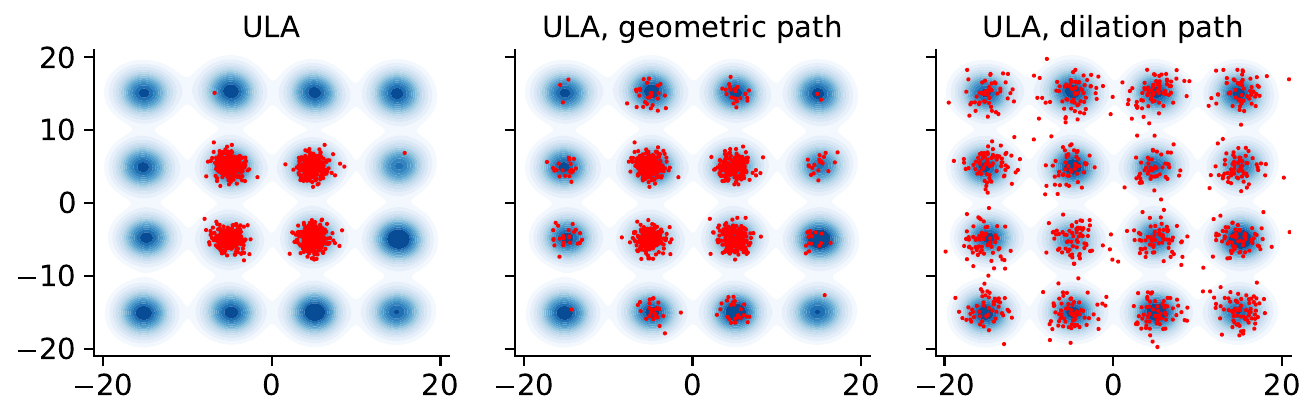} 
    \includegraphics[width=\linewidth]{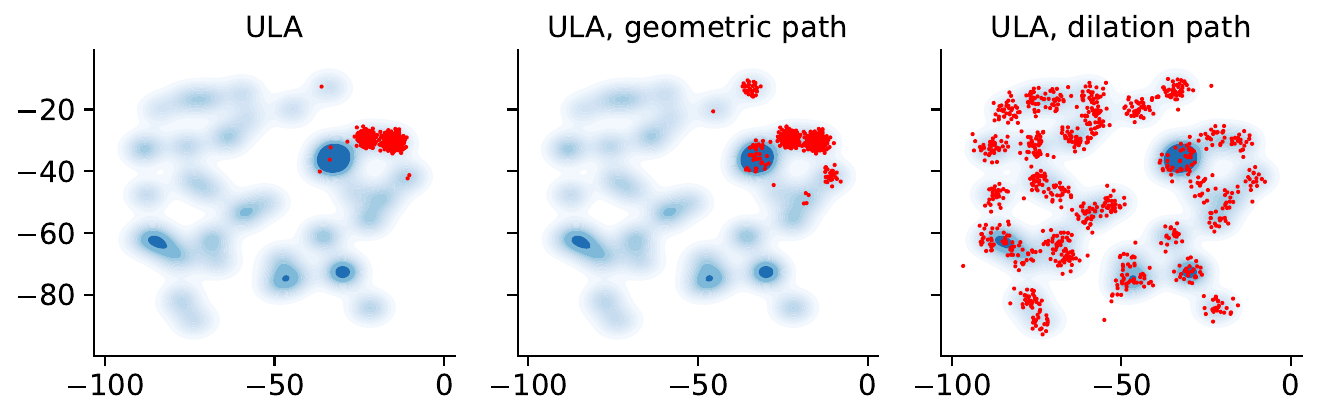} 

    \caption{
    \textit{Top}. 16-mode Gaussian mixture target and standard Gaussian proposal. \textit{Bottom}. 40-mode Gaussian mixture target. The proposal is a standard Gaussian for ULA with/out the geometric path, and is a Dirac for ULA with the dilation path. 
    The kernel density estimate of the target distribution is in blue; particles generated by the sampling process are in red. Simulations involved $1000$ particles, $10 000$ iterations, a step size of $0.001$, and a linear schedule. 
    }
    \label{fig:gaussian_mixtures}
\end{figure}

These experiments highlight known benefits of Annealed Langevin Dynamics that follow a convolutional path, rather than an alternative, geometric path or no path at all. As expected in this setting, ULA with and without the geometric path is stuck for many iterations in modes that are closest to the initialization. One may be tempted to improve consider a proposal distribution with wider coverage so that more modes are already reached at initialization, but without knowing the locations and weights of the modes, the choice of ``large enough" a Gaussian variance would be arbitrary. For example, some modes of the target could always be left in the tails of a proposal with larger variance. In contrast, ULA with the dilation path manages to recover all modes. This is reflected in most metrics in Figure~\ref{fig:fourty_modes_convergence}.

\begin{figure}
\centering
    \includegraphics[width=\linewidth]{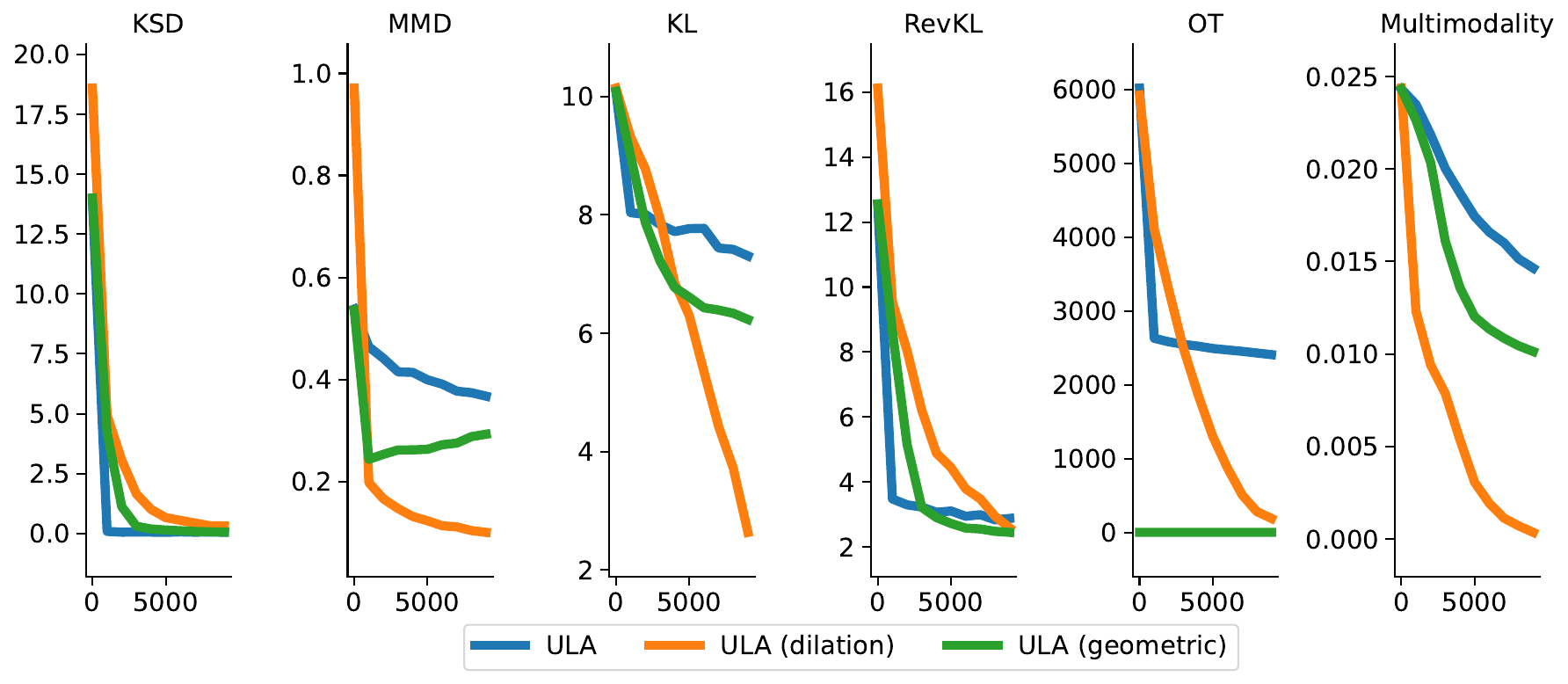} 
    \caption{
    Convergence diagnostics for sampling the 40-mode Gaussian mixture target.
    }
    \label{fig:fourty_modes_convergence}
\end{figure}

\paragraph{Sampling images}
Here, we use the dataset of images MNIST~\citep{deng2012mnist}. Each image has a resolution of $28 \times 28$ pixels and can therefore be understood as a vector in a high-dimensional space $\R^{784}$ with entries in $\llbracket 0, 255 \rrbracket$.
The target distribution is estimated from the MNIST dataset using a score-based diffusion model following~\citet{song2019scorebasedmodel}. Results are reported in Figure~\ref{fig:mnist}. 

\begin{figure}[!ht]
\centering
    \includegraphics[width=\linewidth]{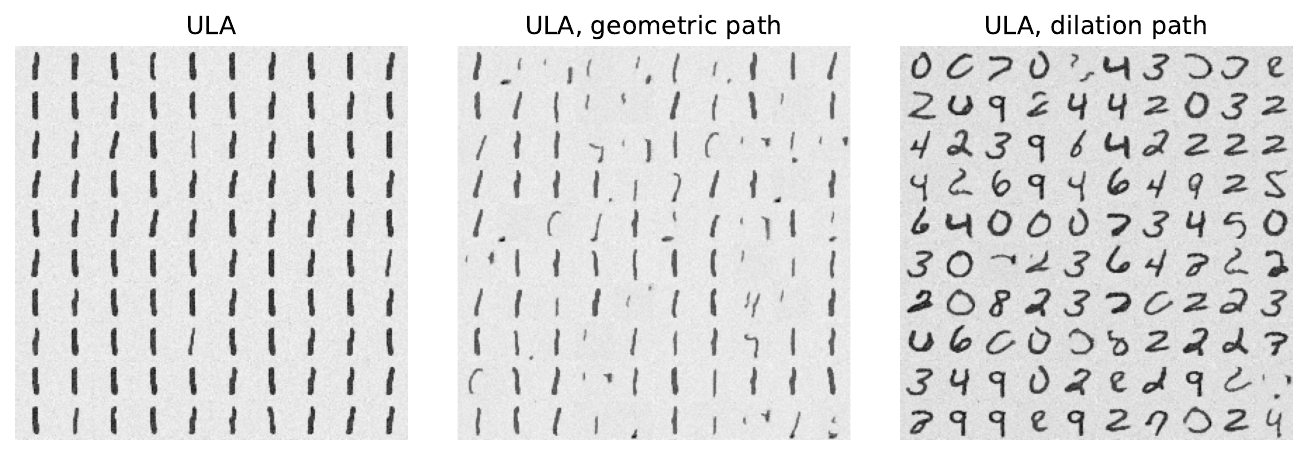} 
    \includegraphics[width=\linewidth]{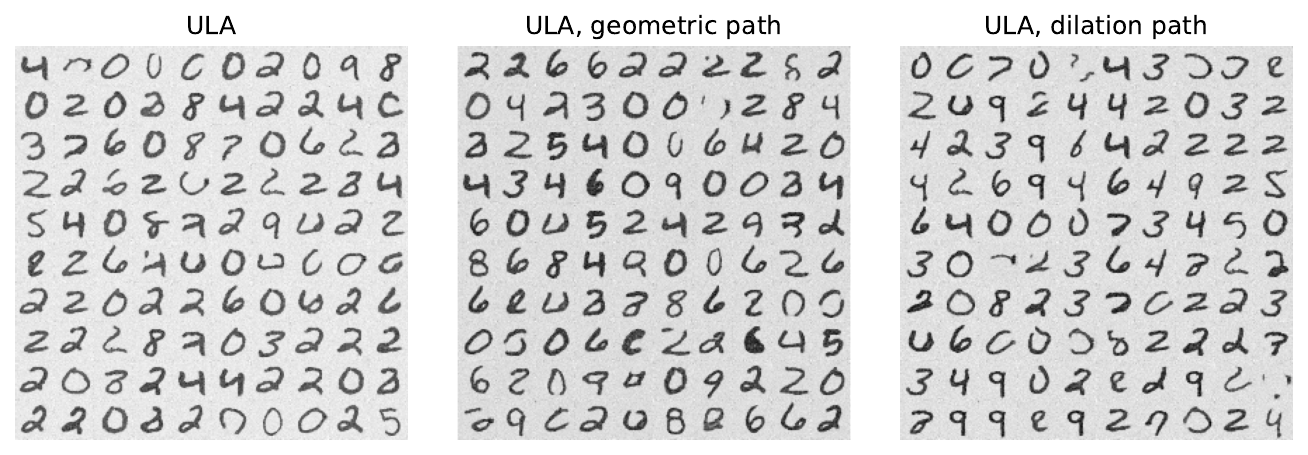} 

    \caption{
    \textit{Top}. The proposal is a standard Gaussian. \textit{Bottom}. The proposal is a uniform distribution over the pixel domain.
    In both cases, the target distribution over images is estimated from the MNIST dataset using a score-based diffusion model following~\citet{song2019scorebasedmodel}. 
    Simulations involved $100$ particles, $500$ iterations, a step size of $0.001$, and a linear schedule. 
    }
    \label{fig:mnist}
\end{figure}

ULA with the dilation path consistently finds many modes, corresponding to different digits. The other sampling schemes, ULA with/out the geometric path, find modes in different proportions depending on the initialization: for example, when initialized near the origin at the top of Figure~\ref{fig:mnist}, they predominantly find the mode for digit ``one", which is closest to the origin as we verify in Figure~\ref{fig:mnist_distances}. However, with a uniform initialization, ULA with/out the geometric path finds more modes in more balanced proportions. A careful, quantitive study to understand how the mode proportions are affected by the initialization is left for future work.

\section{Discussion}

In this work, we introduced the dilation path, which a limit case of the popular convolutional path where the score vectors are available in closed-form. We show using this path for Annealed Langevin dynamics with a step size that is adaptive to both the time and position of a particle, yields an efficient and easy to implement sampler for multi-modal distributions. While we verified that this sampler recovers the locations of the modes (mode ``coverage"), future work will be needed to verify if it correctly recovers their shapes as well (mode ``fidelity").


\paragraph{Acknowledgements}
Omar Chehab was supported by funding from the French ANR-20-CHIA-0016.
The authors thank Nina Vesseron for sharing some code for Langevin sampling of images
and Adrien Vacher for interesting conversations on the geometry of the dilation path. 

\vskip 0.2in
\bibliography{references}

\newpage

\appendix

\clearpage
\newpage
\onecolumn

\section{Computational complexity the recursive algorithm}
\label{app:sec:complexity_calculation}

We here discuss the computational complexity of the recurive algorithm of~\citet{huang2024reversediffusionrecursive}. We would like to approximate $\nabla \log \mu_k(\vx)$. We can do so using $N_p$ particles sampled from Eq.~\ref{eq:monte_carlo_sampling_distribution} using ULA with $N_i$ iterations. This will involve $N_p \times N_i$ evaluations of the target score $\nabla \log \pi(\cdot)$. Each of these scores can be approximated by repeating this procedure over $N_s$ segments. The total computational complexity is therefore $(N_p \times N_i)^{N_s}$ queries of the final target score. This is exponential in the number of segments, which is computationally prohibitive.

\section{Useful Lemma}
\label{app:sec:lemma}

\begin{lemma}[Useful identities for a Gaussian density]
\label{lemma:gaussian_transformations}
Interchangeability of mean and variance: $\mathcal{N}(\vx; \vmu, \mSigma) = \mathcal{N}(\vmu; \vx, \mSigma)$.
\\
Shift $\va \in \R^d$,
$\mathcal{N}(\vx - \va; \vmu, \mSigma) = \mathcal{N}(\vx; \vmu + \va, \mSigma)$.
\\
Scaling $a \in \R$,  
$\mathcal{N}(\vx / a; \vmu, \mSigma) = a \mathcal{N}(\vx; a \vmu, a^2 \mSigma)$.
\\
Gradient $\nabla \mathcal{N}(\vx; \vmu, \mSigma) = -\Sigma^{-1}(\vx - \vmu) \times \mathcal{N}(\vx; \vmu, \mSigma)$.
\\
Product $\mathcal{N}(x; \mu_1, \sigma_1^2) \times \mathcal{N}(x; \mu_2, \sigma_2^2)= \mathcal{N}(x; \frac{\mu_1 \sigma_1^2 + \mu_2 \sigma_2^2}{\sigma_1^2 + \sigma_2^2}, \frac{\sigma_1^2 \sigma_2^2}{\sigma_1^2 + \sigma_2^2})$
\\
Convolution $\mathcal{N}(x; \mu_1, \sigma_1^2) \times \mathcal{N}(x; \mu_2, \sigma_2^2)= \mathcal{N}(x; \mu_1 + \mu_2, \sigma_1^2 + \sigma_2^2)$.
\end{lemma}

\section{Parametric families}
\label{app:parametric_families}

We first prove Proposition~\ref{proposition:gaussian_mixture}, restated here
\begin{proposition*}[Gaussian mixture parametric family]    
    \label{proposition:gaussian_mixture}
    Suppose the proposal is a Gaussian $\nu := \mathcal{N}(\cdot, \vzero, \mSigma_0)$  and the target is a Gaussian mixture with $M$ modes with means $(\vmu_m)_{m \in \llbracket 1, M \rrbracket}$, covariances $(\mSigma_m)_{m \in \llbracket 1, M \rrbracket}$, and positive weights $(w_m)_{m \in \llbracket 1, M \rrbracket}$ that sum to one.
    
    Then, distributions $\pi_t$ along the convolutional path remain in the Gaussian mixture parametric family. Their weights are constant
    $w_m(t) = w_m$ and their means and covariances interpolate additively the proposal's and target's, as $\vmu_m(t) = \sqrt{t} \vmu_m$ and $\mSigma_m(t)^2 = (1-t) \mSigma_0 + t \mSigma_m$. 
\end{proposition*}
\begin{proof}
Distributions along the convolutional path between the target to the proposal, are given by
\begin{align*}
    p_{\lambda_t}(\vx)
    &=
    \frac{1}{\sqrt{1-\lambda_t}} 
    \nu \bigg(
    \frac{\vx}{\sqrt{1-\lambda_t}} 
    \bigg) 
    * 
    \frac{1}{\sqrt{\lambda_t}} 
    \pi \bigg(
    \frac{\vx}{\sqrt{\lambda_t}}
    \bigg)
    \\
    &= 
    \frac{1}{\sqrt{1-\lambda_t}}
    \mathcal{N} \bigg(
    \frac{\vx}{\sqrt{1-\lambda_t}}; \vzero, \mSigma_0
    \bigg)
    \ast 
    \sum_{i=1}^M
    w_m
    \frac{1}{\sqrt{\lambda_t}}
    \mathcal{N} \bigg(
    \frac{\vx}{\sqrt{\lambda_t}}; \vmu_m, \mSigma_m
    \bigg)
    \\
    &=
    \mathcal{N}(\vx; \vzero, (1-\lambda_t) \mSigma_0)
    * 
    \sum_{i=1}^M
    w_m
    \mathcal{N}(\vx; \sqrt{\lambda_t} \vmu_m, \lambda_t \mSigma_m)
    \\
    &=
    \sum_{i=1}^M
    w_m
    \mathcal{N}(\vx; \vzero, (1-\lambda_t)\mSigma_0)
    *
    \mathcal{N}(\vx; \sqrt{\lambda_t} \vmu_m, \lambda_t \mSigma_m)
    \\
    &=
    \sum_{i=1}^M
    w_m
    \mathcal{N}(\vx;  \sqrt{\lambda_t} \vmu_m, (1-\lambda_t) \mSigma_0 + \lambda_t \mSigma_m)
    \enspace .
\end{align*}
They conveniently remain in the Gaussian mixture parametric family, and their parameters are
\begin{align}
    w_m(t) = w_m
    \quad
    \vmu_m(t) 
    = 
    \sqrt{\lambda_t} \vmu_m
    \quad
    \mSigma_m(t) 
    = 
    (1 - \lambda_t) \mSigma_0 
    + 
    \lambda_t \mSigma_m
    \enspace .
\end{align}
If we consider the special case with the proposal covariance is $\mSigma_0 = \epsilon \mI$, then in the limit $\epsilon \rightarrow 0$ the proposal becomes a Dirac and we recover the dilation path with parameters
\begin{align}
    w_m(t) = w_m,
    \quad 
    \vmu_m(t) 
    = 
    \sqrt{\lambda_t} \vmu_m,
    \quad
    \mSigma_m(t) = 
    \lambda_t \mSigma_m
    \enspace .
\end{align}
\end{proof}

\section{Experiments}
\label{app:sec:experiments}

\subsection{Statistical divergences used for evaluation}
\label{app:ssec:statistical_divergences}

We next recall the definitions of the statistical divergences we use in our experiments, to measure the discrepancy between the distribution of particles $p_k$ realizing the sampling process and the target distribution $\pi$. 
\begin{itemize}
    \item Kernel Stein Discrepancy (KSD)~\citep{liu2016ksd,chwialkowski2016ksd,mackey2017ksd}

    Computing the KSD requires access to samples from $p_k$ and to the density (more specifically, the score) of $\pi$. 
    
    This statistical divergence is defined as
    \begin{align}
        \mathrm{KSD}^2(p_k, \pi)
        & = 
        \E_{(\vx, \vx') \sim 
        p_k \otimes p_k} [
        K(\vx, \vx'; \pi, K')
        ]
    \end{align}
    where $K$ is a kernel whose computation requires the (unnormalized) target density $\pi$ and another kernel $K'$
    \begin{align}
        K(\vx, \vy; \pi, K')
        &=
        \nabla \log \pi(\vx)^T \nabla \log \pi(\vy) K'(\vx, \vy)
        +
        \nabla \log \pi(\vx)^T \nabla_y K'(\vx, \vy)
        \\
        &+
        \nabla_x K'(\vx, \vy)^T \nabla \log \pi(\vy)
        + 
        \nabla_\vx \cdot \nabla_\vy K'(\vx, \vy)
        \enspace .
    \end{align}
    chosen by the user. We use the recommended choice by~\citet{gorham2017imqkernel}, known as the Inverse Multiquadratic Kernel $K'(\vx, \vy) = (1 + \| \vx - \vy \|_2^2)^{-\beta}$ where $\beta \in [0, 1]$ and is here chosen as $0.5$.

    \item Maximum Mean Discrepancy (MMD)~\citep{gretton2012mmd}

    Computing the MMD requires access to samples from both $p_k$ and $\pi$.  
    
    This statistical divergence is defined as
    \begin{align}
        \mathrm{MMD}^2(p_k, \pi)
        & = 
        2 
        \E_{(\vx, \vy) \sim p_k \otimes \pi}
        [K(\vx, \vy)]
        -
        \E_{(\vx, \vx') \sim p_k \otimes p_k}
        [K(\vx, \vx')]
        -
        \E_{(\vy, \vy') \sim \pi \otimes \pi}
        [K(\vy, \vy')]
    \end{align}
    where $K$ is a kernel chosen by the user. We use the Gaussian kernel $K(\vx, \vy) = \exp(-\| \vx - \vy \|_2^2 / 2h)$ with a bandwith $h = 1$. 

    \item Kullback Leibler (KL) and reverse Kullback Leibler (revKL) divergences

    Computing the KL divergence usually requires access to samples from $p_k$ and to the densities of both $p_k$ and $\pi$. The reverse KL divergence is defined by switching the roles of $p_k$ and $\pi$. In our experiments, we use an approximation implemented in the scipy library~\citep{scipy} that instead requires access to samples only, from both $p_k$ and $\pi$. 

    The KL divergence is defined as
    \begin{align}
        \KL(p_k, \pi)
        =
        \E_{\vx \sim p_k} \left[
        \log \frac{p_k(\vx)}{\pi(\vx)}
        \right]
    \end{align}
    and switching the roles of $p_k$ and $\pi$ yields the reverse KL divergence.

    \item 2-Wasserstein distance (OT)

    Computing the OT (which stands for Optimal Transport) requires access to samples from both $p_k$ and $\pi$. In our experiments, we actually do not approximate directly the OT, but an upper bound using the Sinkhorn algorithm implemented in the ott library~\citep{cuturi2022ott}. 

    The 2-Wasserstein distance is defined as
    \begin{align}
        W_2(p_k, \pi)
        =
        \inf_{c \in \mathcal{C}}  E_{(\vx, \vy) \sim c} \left[
        \| \vx - \vy \|^2
        \right]
    \end{align}
    where $\vx \sim p_k$ and $\vy \sim \pi$ and $\mathcal{C}$ is the set of joint distributions over such $(\vx, \vy)$. 
        
\end{itemize}
Last, we use the Multimodality Score (MMS) defined as the root mean squared error between the actual and expected number of particles per mode. 

\begin{figure}[!ht]
\centering
    \includegraphics[width=0.5\linewidth]{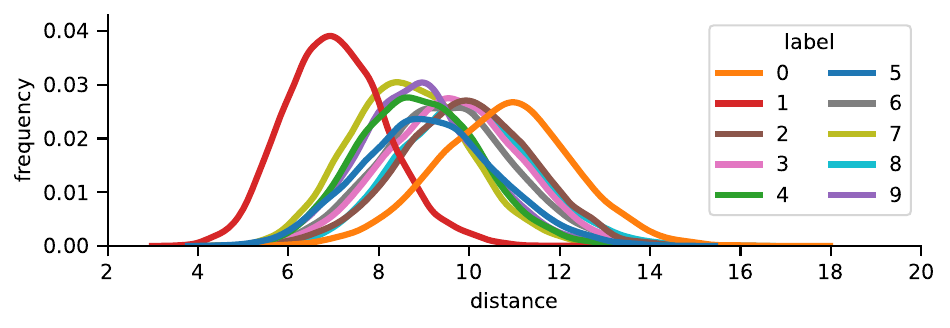} 
    \caption{
    Distribution of the Euclidean distances of image vectors to the origin for the MNIST train dataset. We use the default kernel density estimate from the Seaborn python library~\citep{seaborn}. 
    }
    \label{fig:mnist_distances}
\end{figure}

\begin{figure}[!tbp]
  \centering
  \begin{minipage}[b]{0.45\textwidth}
    \includegraphics[width=\linewidth]{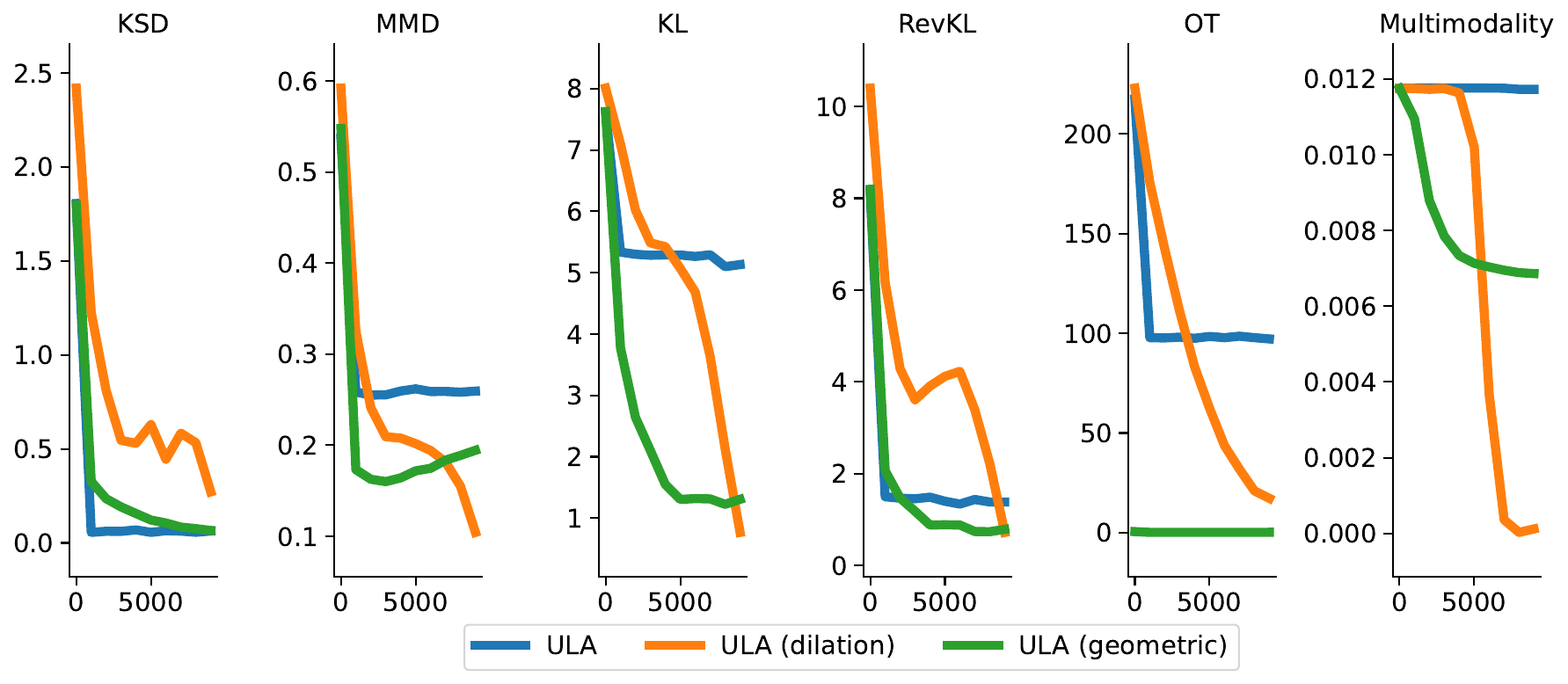} 
  \end{minipage}
  \hfill
  \begin{minipage}[b]{0.45\textwidth}
    \includegraphics[width=\linewidth]{image/fourty_modes_convergence.pdf} 
  \end{minipage}
    \label{fig:sixteen_and_fourty_modes_convergence}
    \caption{
    \textit{Left}. 
    Convergence diagnostics of the sixteen modes experiment. These divergences seem to be more sensitive to mode ``fidelity" than mode ``coverage", given that ULA seems to do better than ULA dilation. 
    \textit{Right}.     
    Convergence diagnostics of the fourty modes experiment. The MMD and KL divergences seem to be sensitive to mode ``coverage" where ULA dilation does best visually.
    }
\end{figure}

\end{document}